\documentclass[letterpaper]{article}

\usepackage{aaai}
\usepackage{times} 
\usepackage{helvet} 
\usepackage{courier} 

\usepackage{color}
\usepackage{amsmath}
\usepackage{amssymb}
\usepackage{amsthm}
\usepackage{hyperref}
\usepackage[ruled]{algorithm}
\usepackage[noend]{algpseudocode}
\usepackage{graphicx}
\usepackage{colortbl}
\usepackage{arydshln}

\newcommand{\diff}[1] {{\Large\color{blue} #1}}

\newtheorem{lmm}{Lemma}

\title{Justifying and Improving Meta-Agent Conflict-Based Search}
\author {David Tolpin, \\
dtolpin@robots.ox.ac.uk}

\begin{document}

\maketitle

\begin{abstract}
The Meta-Agent Conflict-Based Search~(MA-CBS) is a recently proposed
algorithm for the multi-agent path finding problem. The algorithm is
an extension of Conflict-Based Search~(CBS), which automatically
merges conflicting agents into meta-agents if the number of conflicts
exceeds a certain threshold. However, the decision to merge agents is
made according to an empirically chosen fixed threshold on the number
of conflicts. The best threshold depends both on the domain and on the
number of agents, and the nature of the dependence is not clearly
understood.

We suggest a justification for the use of a fixed threshold on the
number of conflicts based on the analysis of a model problem. Following
the suggested justification, we introduce new decision policies for
the MA-CBS algorithm, which considerably improve the algorithm's
performance. The improved variants of the algorithm are evaluated on
several sets of problems, chosen to underline different aspects of the
algorithms.
\end{abstract}

\section{Introduction}

In the Multi-Agent Path Finding~(MAPF) problem, we are given a graph
$G(V,E)$ and a set of $N$ agents $a_1 ... a_N$. Each agent $a_i$ has a
start position $s_i\in V$ and a goal position $g_i\in V$. At each time
step an agent can either move to a neighboring location or wait in its
current location, at some cost. The objective is to return a
least-cost set of actions for all agents, which will move all of the
agents from start to goal positions goal without conflicts (i.e.,
without any pair of agents {\it being in the same node} or {\it
  crossing the same edge} at the same time). MAPF has practical
applications in robotics, video games, aviation, vehicle routing, and other
domains~\cite{silver2005coopeartive,conf/socs/WangBK11}. In its
general form, MAPF is NP-complete, since it is a generalization of
the sliding tile puzzle, an NP-complete problem~\cite{conf/aaai/RatnerW86}.

In this paper we consider a particular variant of MAPF, for which
Meta-Agent Conflict-Based Search (MA-CBS)~\cite{Sharon.macbs}, the
algorithm explored here, was formulated. The total solution cost is the
sum of costs of all actions (and hence the sum of costs of solutions
for each of the agents). Any single action, as well as waiting during
a single time step in a non-goal position, has unit cost. Waiting in
the goal position has zero cost. The problem is solved in {\it
the centralized computing} setting, where a single program controls all
of the agents\footnote{This setting is tantamount to decentralized
  cooperative setting with full knowledge sharing and free
  communication~\cite{Sharon.macbs}.}.

MA-CBS is a generalization of Conflict-Based Search
(CBS)~\cite{conf/aaai/SharonSFS12}.  MA-CBS may serve as a bridge
between CBS and completely coupled solvers, such as A*,
A*+OD~\cite{conf/aaai/Standley10}, or
EPEA*~\cite{conf/aaai/FelnerGSSBSSH12}.  MA-CBS starts as a regular
CBS solver, where the low-level search is performed by a single-agent
search algorithm. At every search step MA-CBS employs a {\it
  heuristic:} if the number of conflicts for a pair of agents exceeds
a certain threshold $B$, MA-CBS merges the two agents into a combined
agent. Experimental results showed that for certain values of the
threshold MA-CBS outperforms both CBS and single-agent
search. However, threshold $B$ used in the heuristic has to be
empirically determined, and varies both with the size and shape of
graph $G$ and with the number of agents $N$.  Difficulty choosing the
`right' value for $B$ limits practical usability of MA-CBS.

Generally, a heuristic represents abstraction or approximation
of a phenomenon associated with the problem or
algorithm. Understanding why a particular heuristic works
helps make better decisions involving the heuristic. One
way to discover powerful heuristics for a particular problem
is to design them systematically~\cite{Holte.steps}. However, a
heuristic can also come as an insight, and in this case
explaining why the heuristic is successful helps further improve
the algorithm.

In this paper we look at the heuristic decision-making of MA-CBS,
in which a fixed threshold on the number of conflicts between a pair
of agents is used to replace the agents with a single combined agent.
Based on the observations of the dependence of the threshold on
features of the problem, we suggest an explanation for the threshold,
and propose a model problem where the decision can be made optimal
in a certain sense of optimality. Based on the model problem, we
empirically investigate variants of MA-CBS. The investigation
\begin{itemize}
\item provides further support for the hypothesis regarding the root
  cause behind the fixed threshold, and
\item allows improving MA-CBS algorithm through better use of the
  heuristic.
\end{itemize}
Finally, we propose more efficient decision rules, based both on
well-known and new results, for merging agents
in MA-CBS. We empirically compare the variants of MA-CBS on
different problem domains to illustrate a steady increase of
performance.

\textbf{Contributions} of the paper are:
\begin{itemize}
\item Justification of heuristic decision-making in MA-CBS.
\item Improvement of MA-CBS based on understanding of the phenomenon
  behind the use of a fixed threshold.
\item Derivation of variants of MA-CBS with further improved
  performance.
\end{itemize}

\section{Background and Related Work}

The pseudocode for MA-CBS is shown in Algorithm~\ref{alg:ma-cbs}. Like
CBS, MA-CBS maintains a list of nodes, sorted by the increasing sum of
costs of individual solutions (SIC). At every step of the main loop
(lines~\ref{alg:ma-cbs-loop}--\ref{alg:ma-cbs-end-loop}) a node with
the lowest SIC is removed from the node list
(line~\ref{alg:ma-cbs-pop}). If the solutions in the node do not have
any conflicts, this set of solutions is returned as the solution for
the problem~(line~\ref{alg:ma-cbs-solutions}). In case of conflicts
there are two possibilities. MA-CBS either adds, just like CBS, two
nodes to the node list. The nodes are created according to a single
conflict between a pair of agents. Each of the nodes has the solution
for one of the agents updated to avoid the conflict with the other
agent~(lines~\ref{alg:ma-cbs-split}--\ref{alg:ma-cbs-end-split}). Otherwise,
MA-CBS merges the two agents into a combined agent and adds a single
node to the node list with the combined agent instead of the pair of
agents~(lines~\ref{alg:ma-cbs-merge}--\ref{alg:ma-cbs-end-merge}). The
decision whether to split or to merge is based on parameter $B$: the
agents are merged if the number of encountered conflicts between the
agents since the beginning of the search is at least $B$.
\begin{algorithm}[h!]
\caption{MA-CBS}
\label{alg:ma-cbs}
\begin{algorithmic}[1]
\Procedure {MA-CBS}{Agents, B}
\State Nodelist $\gets$ [\Call{Node}{Agents}]
\Loop \label{alg:ma-cbs-loop}
   \If {\Call{Empty?}{Nodelist}} {~\Return FAILURE}
   \Else
     \State Node $\gets$ \Call{Pop}{Nodelist} \label{alg:ma-cbs-pop}
     \If {\Call{Conflicts?}{Node}}
       \If {\Call{Merge?}{Node,B}}
         \State Node$'$ $\gets$ \Call{Merge}{Node} \label{alg:ma-cbs-merge}
         \State \Call{\sc Insert}{Node$'$, Nodelist} \label{alg:ma-cbs-end-merge}
       \Else
         \State Node$'$, Node$''$ $\gets$ \Call{Split}{Node} \label{alg:ma-cbs-split}
         \State \Call{Insert}{Node$'$, Nodelist}
         \State \Call{Insert}{Node$''$, Nodelist} \label{alg:ma-cbs-end-split}
       \EndIf
     \Else {~\Return \Call{Solutions}{Node}} \label{alg:ma-cbs-solutions}
     \EndIf
   \EndIf
\EndLoop \label{alg:ma-cbs-end-loop}
\EndProcedure
\end{algorithmic}
\end{algorithm}

Both CBS and MA-CBS solve MAPF optimally, however sub-optimal variants
of CBS were also introduced~\cite{conf/socs/BarrerSSF14}. On the
other hand, different algorithms for solving MAPF optimally are also
pursued. Some of the other algorithms bear similarities to MA-CBS,
such as Independence Detection (ID)~\cite{conf/aaai/Standley10}, which
for every pair of conflicting agents tries to find an alternative
solution for each agent avoiding the conflicts, and if failed
merges the conflicting agents into a combined agent. A suboptimal
variant of ID offers a trade-off between running time and solution
quality~\cite{conf/ijcai/StandleyK11}. Other algorithms, such as
A*+OD~\cite{conf/aaai/Standley10}, EPEA*\cite{conf/aaai/FelnerGSSBSSH12},
or ICTS~\cite{journals/ai/SharonSGF13} can be used for
lower-level search in MA-CBS.

\section{Justification of Fixed Threshold}

The authors of MA-CBS summarized the results of their empirical
evaluation of the algorithm with an evidence that:
\begin{itemize}
\item The best value of threshold $B$ decreases with hardness of the
  problem instances.
\item The advantage of MA-CBS is more prominent on harder instances.
\end{itemize}
Such behavior is characteristic for \emph{online competitive
  algorithms}~\cite{conf/stoc/ManasseMS88}, and in particular reminds of
the \textit{ski rental problem}, also known as \textit{snoopy caching
  problem} in a more general
setting~\cite{journals/algorithmica/KarlinMRS88}. In the ski rental
problem a tourist at a ski resort may either pay a fixed rent for each
day of ski rental, or to buy the ski, obviously at a higher price. The
tourist does not know in advance how many days he is going to spend at
the resort, and must decide every morning whether to rent or to buy. The
famous result for this problem is that the tourist should rent the ski
for $\frac {\mbox{ski price}} {\mbox{daily rent}} - 1$ days, and to
buy the ski on the next day if he/she is still at the resort. This
algorithm is 2-competitive, that is the tourist will spend at most
twice as much money as if the number of days were known in advance,
and this is the best competitive ratio a deterministic algorithm can
achieve. There are randomized algorithms with a lower competitive
ratio though~\cite{Karlin.randomized}.

Consequently, we conjectured that the fixed threshold in MA-CBS plays
a role similar to the threshold in the online algorithm for the ski
rental problem. Both theoretical analysis and empirical evaluation
confirmed this conjecture.

\subsection{Model Problem: 2 agents}

\begin{figure}[t]
\centering
\begin{tabular}{c c}
\includegraphics[scale=0.9]{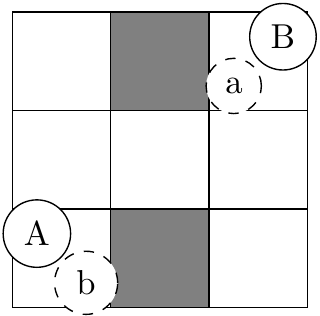} 
& \includegraphics[scale=0.9]{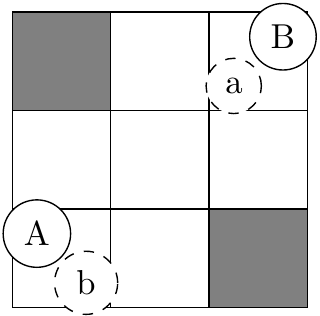} \\
a & b
\end{tabular}
\caption{Scenes with 2 agents.}
\label{fig:bat}
\end{figure}

Consider the MAPF problem for 2 agents as the simplest non-trivial
case. If MA-CBS is used, and the number of conflicts reaches $B$, some
number of merges between 1 and the number of nodes currently in the
node list solves the problem instance. If the time to find a solution
for the combined agent does not become much shorter when constraints
are added, it may be better to just remove all constraints and compute
the solution for the combined agent \emph{once}, rather than multiple
times for each node in the node list. We shall call a version of
MA-CBS that restarts the search upon a merge MA-CBS/R. A comparative
evaluation of MA-CBS and MA-CBS/R is provided in
Table~\ref{tbl:bat-1-R}. The problem instance is shown in
Figure~\ref{fig:bat}.a. The number of merges performed by MA-CBS is, for
all but extreme, (1 and 8) values of $B$ is greater than 1 (the number
of restarts in MA-CBS/R), and the number of single-agent nodes
expanded by MA-CBS is greater than by MA-CBS/R.\footnote{Let us
  note that the number of expanded single-agent nodes is, along with
  the search time, an adequate measure of the performance of CBS,
  MA-CBS, and variants. Evaluation of the distance heuristic for a
  single agent can be memoized, and the total heuristic evaluation
  time is thus negligible compared to the time spent expanding
  single-agent nodes and generating children satisfying the
  constraints.}

\begin{table}[h!]
\centering
\begin{tabular}{ r | r | r | r }
  & & \multicolumn{2}{c}{{\bf nodes}}\\ \cline{3-4}
  &&&\\[-8pt]
B & {\bf merges} & {\bf MA-CBS} & {\bf MA-CBS/R} \\
1 & 1 & 66 & 66 \\
2 & 2 & 136 & 80 \\
3 & 3 & 207 & 95 \\
4 & 4 & 278 & 110 \\
5 & 3 & 238 & 126 \\
6 & 2 & 198 & 142 \\
7 & 1 & 158 & 158 \\
8+ & 0 & 118 & 118
\end{tabular}
\caption{MA-CBS/R vs MA-CBS for scene~\ref{fig:bat}.a.}
\label{tbl:bat-1-R}
\end{table}

The intuition behind MA-CBS/R is formalized by the following two
lemmas about competitiveness of both MA-CBS/R and MA-CBS
for 2 agents:
\begin{lmm}
  Let us denote by $T_2$ the time to find the shortest path for the
  combined agent, and by $T_{1,1}$ the time to find the shortest paths
  for both agents independently, ignoring conflicts between the
  agents. Under the assumptions that
  \begin{itemize}
  \item[a)] $T_{1,1}$ and $T_2$ are constant for a given problem
    instance at any point of the algorithm, 
  \item[b)] $T_{2}\ge T_{1,1}$, and
  \item[c)] the ratio $\frac {T_2} {T_{1,1}}$ is known in advance,
  \end{itemize}
  MA-CBS/R is $2-\frac 1 B$-competitive, and the competitive ratio
  is achieved for $B=\lfloor \frac {T_2} {T_{1,1}} \rfloor$.
  \label{lmm:ma-cbs-r-2-comp}
\end{lmm}

\begin{proof} Since merging two agents solves a 2-agent problem
  at the cost $T_2$, and splitting on a conflict may or
  may not solve the problem at the cost $T_{1,1}$, this problem is
  equivalent to the ski rental or two caches and one block snoopy
  caching problem~\cite{journals/algorithmica/KarlinMRS88}.
\end{proof}

\begin{lmm}
  Under the assumptions of Lemma~\ref{lmm:ma-cbs-r-2-comp}
  MA-CBS is $1+B-\frac 1 B$-competitive, and the competitive ratio
  is achieved for $B=\lfloor \frac {T_2} {T_{1,1}} \rfloor$.
  \label{lmm:ma-cbs-2+B-comp}
\end{lmm}

\begin{proof} After $k$ splits there are $k+1$ nodes in the node list
(Algorithm~\ref{alg:ma-cbs}) for any $k\ge 0$. Hence, MA-CBS performs
at most $B-1$ splits and then at most $B$ merges, and the
worst-case time is $T_{MA-CBS}=BT_2+(B-1)T_1$. Just like in
the proof for the ski rental problem, the competitive ratio is
\begin{equation}
\min_B\min\left(\frac {T_{MA-CBS}}
    {T_2},\frac{T_{MA-CBS}}{BT_1}\right)=1+B-\frac 1 B
\label{eqn:c-ma-cbs}
\end{equation}
for $B=\lfloor \frac {T_2}
{T_1} \rfloor$.
\end{proof}

According to the assumptions of Lemma~\ref{lmm:ma-cbs-r-2-comp},
$\frac {T_2} {T_{1,1}}$ is at least 1, hence MA-CBS/R is competitive
with a lower ratio (that is, in the worst case finds a solution in a
shorter time) than MA-CBS.

The worst-case approach is apparently a reasonable option for
designing an algorithm for 2-agent MAPF. Table~\ref{tbl:bat-01} shows
solution costs and the amount of computation spent by CBS and MA-CBS/R
to find the solutions for two problem instances in
Figure~\ref{fig:bat}. Both instances have  agents at the same
locations, as well as the same number of passable cells, and the same
position of the bottleneck. Nonetheless, the cost of an optimal
solution for the instance in Figure~\ref{fig:bat}.a is 11, and CBS has
to resolve 7 conflicts, but for the instance in Figure~\ref{fig:bat}.b
the cost is 9, and only 1 conflict has to be resolved before a
solution is found. MA-CBS/R is more efficient for \ref{fig:bat}.a but
not for \ref{fig:bat}.b, where CBS is faster.

\begin{table}[h!]
\centering
\begin{tabular}{c| r | r | r |  r |  r}
 & &\multicolumn{2}{c|}{\bf CBS}&\multicolumn{2}{c}{{\bf MA-CBS/R($\pmb{B=1}$)}}  \\  \cline{3-6}
{\bf scene}&{\bf cost}& {\bf nodes} & {\bf splits} & {\bf nodes} & {\bf restarts} \\
\ref{fig:bat}.a  & 11 & 118 & 7 & 66 & 1 \\
\ref{fig:bat}.b  & 9 & 20 & 1 & 24 & 1
\end{tabular}
\caption{Number of expanded nodes.}
\label{tbl:bat-01}
\end{table}

\section{MA-CBS/R for Any Number of Agents}

MA-CBS/R can be extended to an arbitrary number of agents.
The pseudocode of MA-CBS/R is shown in
Algorithm~\ref{alg:ma-cbs-r}. MA-CBS/R differs from MA-CBS
(Algorithm~\ref{alg:ma-cbs}) in
lines~\ref{alg:ma-cbs-r-restart}--\ref{alg:ma-cbs-r-end-restart}.
Firstly, {\sc Merge/R} creates a node with unconstrained solutions for
individual agents. Secondly, the node list is
re-initialized to contain just the new
node~(line~\ref{alg:ma-cbs-r-end-restart}). Effectively, the search
is restarted with the two agents replaced by a combined agent.
\begin{algorithm}[h!]
\caption{MA-CBS/R}
\label{alg:ma-cbs-r}
\begin{algorithmic}[1]
\Procedure {MA-CBS/R}{Agents,B}
\State Nodelist $\gets$ [\Call{Node}{Agents}]
\Loop
   \If {\Call{Empty?}{Nodelist}}{~\Return FAILURE}
   \Else
     \State Node $\gets$ \Call{Pop}{Nodelist}
     \If {\Call{Conflicts?}{Node}}
       \If {\Call{\diff{Restart?}}{Node, B}} \label{alg:ma-cbs-r-restart}
         \State Node$'$ $\gets$ \Call{\diff{Merge/R}}{Node} \label{alg:ma-cbs-r-merge}
         \State \diff{Nodelist $\gets$ [Node$'$]} \label{alg:ma-cbs-r-end-restart}
       \Else
         \State Node$'$, Node$''$ $\gets$ {\sc Split}(Node)
         \State \Call{Insert}{Node$'$, Nodelist}
         \State \Call{Insert}{Node$''$, Nodelist}
       \EndIf
     \Else {~\Return \Call{Solutions}{Node}}
     \EndIf
   \EndIf
\EndLoop
\EndProcedure
\end{algorithmic}
\end{algorithm}

The decision whether to merge two agents and restart the search is
again based on a fixed threshold. Given the suggested interpretation
for $B$ as an estimate of $\frac {T_{2}} {T_{1,1}}$, merging
\emph{combined}, instead of single, agents into a larger yet agent
should require a threshold that depends on the sizes of the agents to
be merged. This was confirmed by preliminary experiments on partial
sliding tile puzzle (see below), which showed that using {\it the
  same} $B$ for merging both single and combined agents, as in the
original version of MA-CBS~\cite{Sharon.macbs}, \emph{slows down} the
search compared to merging just single agents. Indeed, the number of
children grows exponentially with the number of single agents in a
combined agent, and thus the search time grows at least exponentially
with the size of the combined agent, demanding a higher $B$. In the
experiments\footnote{The program code and the problem instances for
  the experiments in this paper are attached to the submission. In the
  camera-ready version a URL pointing at a public source code
  repository will be provided instead.}, we limited the maximum size
of a combined agent to 2, that is, only single agents would be merged,
efficiently setting $B=\infty$ for producing combined agents
consisting of more than 2 single agents. A more advanced
implementation would be based on different values of $B$ for different
sizes of agents to be merged.

\subsection{Exploring MA-CBS/R with Partial Sliding Tile Puzzle}

The sliding tile puzzle problem is quite obviously an example of MAPF,
and the NP-completeness of MAPF is shown through reduction to the
sliding tile puzzle~\cite{Sharon.macbs}. However, a modified version
of the puzzle can also be used to empirically explore MAPF algorithms
and, in the case of MA-CBS and MA-CBS/R, to understand the influence
of the number of agents and the threshold $B$ on the search time.

\begin{figure}[t]
\centering
\includegraphics[scale=0.8]{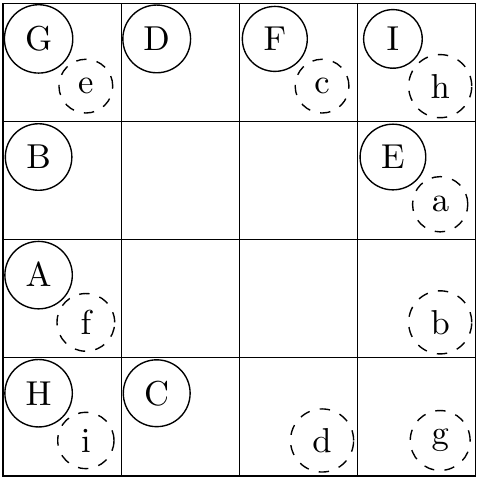}
\caption{A partial sliding tile puzzle instance, 9 tiles.}
\label{fig:4x4-map}
\end{figure}

We used the \emph{partial} sliding tile puzzle, in which only
some of the tiles are present on the $4\times 4$ board, for the
exploration. A problem instance with 9 tiles is shown in
Figure~\ref{fig:4x4-map}. The original locations of the tiles
are marked by solid circles with the uppercase letter
identifying the tile. Dashed circles with lowercase letters are
the goal locations for each of tiles. Any instance with fewer
than 15 tiles is solvable~\cite{JohnsonStory.15puzzle}. The
partial sliding $N$-tile puzzle can be viewed as an MAPF problem
with $N$ agents. A solution of the partial sliding tile puzzle
is translated to a solution of the MAPF problem with each
sequence of upto $N$ moves of different tiles translated to the
simultaneous move of the agents (where some of the agents may be
waiting). Hence, any solvable instance of the puzzle is also
solvable as an MAPF instance.  The solution costs and optimal
solutions can be different though, since waiting in a non-goal
position is free in the partial sliding tile puzzle but not in
the MAPF problem.

\begin{table}[h!]
\centering
\includegraphics[scale=0.5]{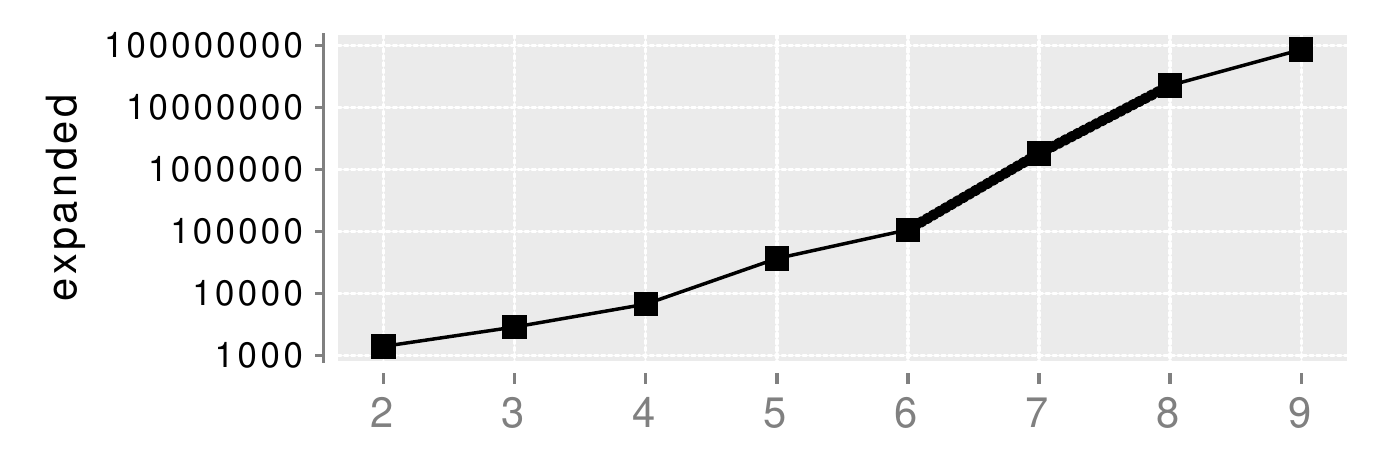}
\begin{tabular}{c | r | r | r}
{\bf agents} & {\bf time, sec} & {\bf expanded} & {\bf split} \\
2 &     1.3 & 1,398 & 18 \\
3 &     2.3 & 2,872 & 91 \\
4 &     4.0 & 6,710 & 349 \\
5 &    13.0 & 36,632 & 2,457 \\ \rowcolor[gray]{0.95}
6 &    33.8 & 106,622 & 7,489 \\ \rowcolor[gray]{0.95}
7 &   518.4 & 1,814,141 & 125,920 \\ \rowcolor[gray]{0.95}
8 & 8,332.3 & 28,832,218 & 1,511,996 \\
9 & 35,428.9 & 86,491,681 & 5,797,029 \\
\end{tabular}
\caption{CBS on $4\times 4$ tile puzzle, amount of computation vs. the number of
  agents.}
\label{tbl:cbs-4x4}
\end{table}

Since solving the sliding tile puzzle optimally requires
problem-specific algorithms~\cite{Korf.BFID85}, we
needed to determine the number of tiles for which problem instances
are hard enough but still solvable using CBS with an A* variant as the
low-level solver. We generated a set of 100 random scenes (agent
locations) of $4\times4$ partial sliding tile puzzle for every number
of agents from 2 to 9, spreading the agents in such a way that
conflicts between the agents are likely. Table~\ref{tbl:cbs-4x4} shows
the total running time of CBS, number of expanded nodes, and number of
splits, for a range of values of $B$.  The growth of the number of
expanded nodes is the highest for 6--8 agents (the bold curve
segment), where we should expect MA-CBS becoming better than
CBS. Indeed, for up to 6 agents CBS is faster than MA-CBS for any $B$
on the problem sets used. However for 7 agents MA-CBS becomes
comparable to CBS, and for 8 agents MA-CBS/R outperforms CBS for a
range of values of $B$, as one can see in Table~\ref{tbl:4x4-8}.

\begin{table}[h!]
\begin{tabular}{ r | r | r | r | r }
{\bf B} & {\bf time, sec} & {\bf expanded} & {\bf split} & {\bf reset} \\
1 & 7,997.4 & 35,214,246 & 569,351 & 381 \\
4 & 9,170.4 & 41,795,136 & 615,817 & 291 \\
19 & 7,011.4 & 30,990,369 & 482,125 & 158 \\
63 & 6,249.0 & 25,826,426 & 480,688 & 115 \\
94 & {\bf 5,480.9} & {\bf 21,999,664} & 424,350 & 105 \\
211 & 6,940.8 & 28,448,940 & 571,912 & 85 \\
317 & 7,047.4 & 28,700,663 & 621,117 & 73
\end{tabular}
\caption{MA-CBS/R on $4\times 4$ tile puzzle, 8 agents.}
\label{tbl:4x4-8}
\end{table}

$B$ can be learned offline on a subset of problem instances. However,
a more efficient approach is to directly estimate the ratio
$\frac{T_2}{T_{1,1}}$ from a few runs of CBS and MA-CBS/R, and to use
$B \approx \frac{T_2}{T_{1,1}}$.  Indeed, this estimate gives
$B \approx 100 \pm 10$ for problem instances with 8 agents, and MA-CBS/R
performs reasonably well in the vicinity of this value of $B$
(Table~\ref{tbl:4x4-8}).

The relative performance of the original MA-CBS is consistent with the
results for 2 agents. Table~\ref{tbl:4x4-8-M} shows the running time,
number of expanded nodes, number of splits and merges, for the same
parameters as Table~\ref{tbl:4x4-8}. The difference between MA-CBS and
MA-CBS/R is more prominent for higher values of $B$, where
MA-CBS/R exhibits much lower running times and numbers of expanded
nodes. Only by $B=317$ the search time and the number of expanded nodes
begin to decrease; this, like in the case of 2 agents, can be
explained by the decrease in the number of search branches reaching
this number of conflicts.
\begin{table}[h!]
\begin{tabular}{ r | r | r | r | r }
{\bf B} & {\bf time, sec} & {\bf expanded} & {\bf split} & {\bf merge} \\
1 & 12,077.5 & 42,593,750 & 876,473 & 39,492  \\
4 & 32,443.3 & 136,769,252 & 2,892,755 & 117,777  \\
19 & 43,633.9 & 170,588,470 & 3,830,660 & 176,188  \\
63 & 46,647.9 & 160,939,149 & 3,669,123 & 182,695  \\
94 & 48,409.4 & 191,118,297 & 4,377,734 & 210,567  \\
211 & 48,993.2 & 199,920,613 & 4,442,505 & 211,900  \\
317 & 40,329.0 & 162,546,522 & 3,861,403 & 167,136 
\end{tabular}
\caption{MA-CBS on $4\times 4$ tile puzzle, 8 agents.}
\label{tbl:4x4-8-M}
\end{table}

\subsection{An Estimate on Competitive Ratio}

An important question is how competitive an MA-CBS/R algorithm can be
in the general case. One answer to the question is the following
lemma, which can be proven by construction.

\begin{lmm}
  Under the assumption that the cost of finding a solution for a
  single meta-agent of any size is known in advance, MA-CBS/R for an
  arbitrary number of agents $N$ can achieve $1 + \lceil \frac N 2 \rceil$
  competitive ratio.
  \label{lmm:ma-cbs-worst-case-ratio}
\end{lmm}

\begin{proof}[Proof outline] MA-CBS/R for an arbitrary number $N$ of agents
merges at most $N$ agents into a combined agent (that is, all the
agents), and performs at least $2^{\lceil \log N \rceil}-1\approx N$
merges. The algorithm should merge two agents every time the total
cost of computations performed from the beginning of the algorithm is
at least the cost of the next merge. Two cases are possible:
\begin{enumerate}
\item After a merge, the total cost of computations performed so far
  is no more than the cost of the next merge.
\item The total cost of computations is greater than the cost of the
  next merge, in which case the agents should be merged
  immediately.
\end{enumerate}
Assuming that the cost of a merge is much higher than the cost of a
split (a basic assumption behind MA-CBS), the second case takes place
when there are several merges of the same cost (that is, when several
combined agents of the same size must be constructed). The number of
such subsequent merges of the same cost is at most $\lceil \frac N 2
\rceil$ since the size of a combined agent is at least 2, and the
total overhead is $1+ \lceil \frac N 2 \rceil$. Thus, an algorithm
with a competitive ratio of $1 + \lceil \frac N 2 \rceil$ can be, at
least theoretically, constructed within the framework of MA-CBS/R.
\end{proof}

For the case of two agents, the competitive ratio $1+ \lceil \frac 2 2
\rceil = 2$ coincides with
Lemma~\ref{lmm:ma-cbs-r-2-comp}.  In practice, since MA-CBS/R uses a 
heuristic suggesting to merge two agents based on the number of
conflicts between \emph{these agents only} rather than the total
number of conflicts encountered, the performance should be better
than what follows from the worst-case analysis.

\section{Further improvements to MA-CBS/R}

There are several directions in which to look for further improvements
in the performance of MA-CBS/R. Here we introduce two improved
variants of MA-CBS/R based on different decision rules whether to
split the search on a conflict or to merge the agents and restart.

The first variant is \textbf{Randomized MA-CBS/R,} which is derived from
the randomized algorithm for the snoopy caching
problem~\cite{Karlin.randomized}. According to the randomized
algorithm, instead of deterministically merging after $B$ conflicts,
the decision to merge can be made randomly for any number $k$ of conflicts
between 1 and $B$ inclusive, and the probability $p_m(k)$ of merging
after $k$ conflicts is given by (\ref{eqn:p_m}):
\begin{equation}
p_m(k) = \left[B \left( \left( \frac {B+1} B \right)^{B-k+1} -1 \right)\right]^{-1}
\label{eqn:p_m}
\end{equation}
One can see that the probability grows with $k$ and reaches 1 for
$k=B$. For the ski rental problem and, consequently, for the 2-agent
case, the randomized algorithm yields a competitive ratio of $\frac e
{e-1} \approx 1.58$~\cite{Karlin.randomized}, compared to the
competitive ratio of $\approx 2$ for deterministic MA-CBS/R.

The second variant is \textbf{Delayed MA-CBS/R,} which decides whether
to merge or to split based both on the number of conflicts for a
particular pair of agents and on the value of the heuristic cost
estimate of the first and the next node in the node list. The algorithm is
based on the observation that the node obtained after a merge is likely
to have a higher heuristic cost estimate, such that instead of
exploring the subtree rooted in the node, the search will
consider other nodes instead. Therefore it might make sense to merge a
pair of agents in a node only if the current cost estimate is
\emph{sufficiently lower} than the cost estimate of the next node in
the node list. 

A full analysis of the utility of merging agents given the cost
estimates of nodes in the node list is beyond the scope of this
paper. However, the simplest case for discrete cost domains (to which
CBS, MA-CBS, and other MAPF algorithms are applied) are two first
nodes with the same cost estimate. In this case it should be
beneficial to delay the merge (and split instead) until the cost
estimate of the first node becomes \emph{strictly lower} than of the
second node, even if the number of conflicts reached $B$. Indeed, this
decision rule was implemented in Delayed MA-CBS/R.

\begin{table}[ht!]
\centering
\begin{tabular}{l | r | r | r | r}
       & {\bf time,sec} & {\bf expanded} & {\bf split} & {\bf reset} \\
MA-CBS/R & 5,480.9 & 21,999,664 & 424,350 & 105 \\
Randomized & 4,984.6 & 20,649,986 & 403,133 & 119 \\
Delayed & 4,252.2 & 17,190,873 & 393024 & 96
\end{tabular}
\caption{Variants of MA-CBS/R on $4\times 4$ partial sliding tile puzzle, 8
  agents, B=94.}
\label{tbl:improved-ma-cbs-r}
\end{table}

The two variants where empirically compared to the basic
MA-CBS/R on $4\times 4$ partial sliding tile puzzle with 8
agents for $B=94$ (M-CBS/R has the best performance for this
$B$)~(Table~\ref{tbl:improved-ma-cbs-r}). Both Randomized and
Delayed MA-CBS/R showed shorter search times and lower numbers
of expanded nodes, by $\approx 10\%$ and $\approx 20\%$
correspondingly. A better yet performance might be achieved
through combining the ideas of both Randomized and Delayed
MA-CBS/R, as well as through a more informed decision-making in
Delayed MA-CBS/R.

\section{Experiments on Benchmark Maps}

\begin{table}[ht!]
\centering
\begin{tabular}{l | r |  r | r}
 & \includegraphics[scale=0.05]{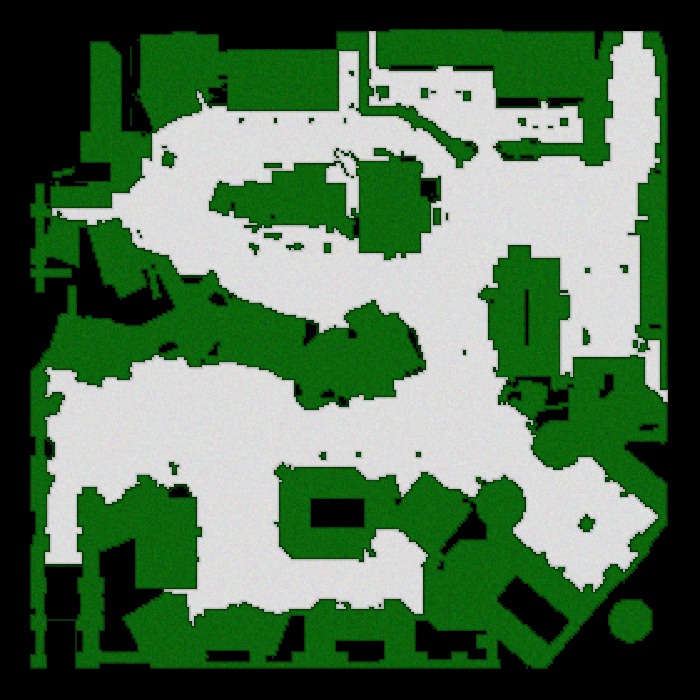}
 & \includegraphics[scale=0.05]{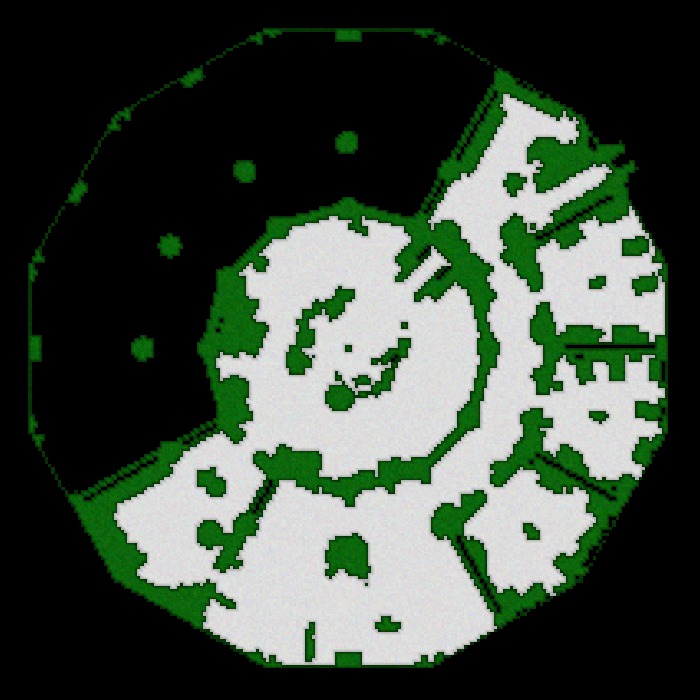}
 & \includegraphics[scale=0.05]{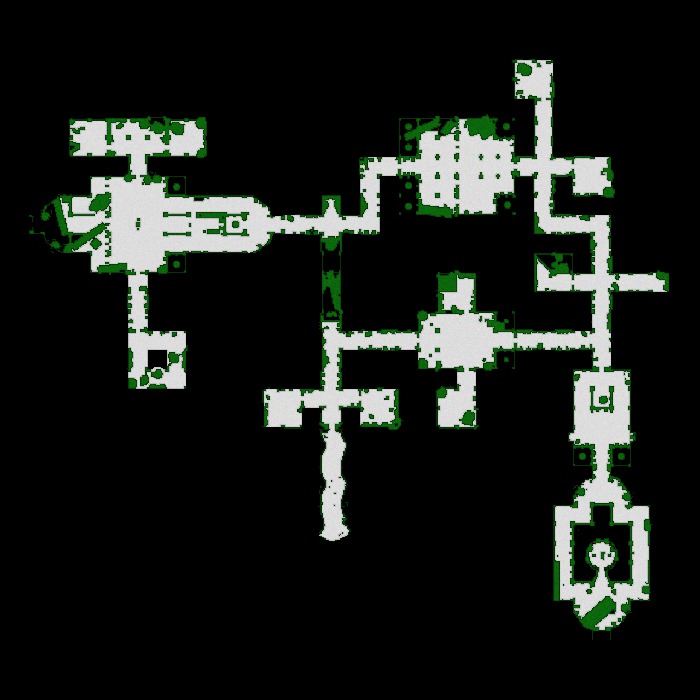} \\
 & {\bf den520d} & {\bf ost003d} & {\bf brc202d} \\ \hline
CBS & 68,321 & 9,800 & 174,727 \\ \rowcolor[gray]{.95}
MA-CBS-R(1) & 3,364 & 7,429 & 97,442 \\ \rowcolor[gray]{.95}
MA-CBS(1) & 2,935 & 7,351 & 98,455 \\ 
MA-CBS-R(16) & {\bf 707} & {\bf 6,837} & {\bf 66,286} \\
MA-CBS(16) & 3,531 & 26,561 & 75,212 \\ \rowcolor[gray]{.95}
MA-CBS-R(64) & 804 & 8,373 & 72,417 \\ \rowcolor[gray]{.95}
MA-CBS(64)   & 6905 & 30,582 & 78,430 \\
MA-CBS-R(256) & 1,454 & 10,296 & 92,086 \\
MA-CBS(256) & 18,704 & 17,844 & 99,987
\end{tabular}
\caption{\textit{Dragon Age: Origins} scenes with 16 agents, search times (sec) for CBS, MA-CBS/R, MA-CBS.}
\label{tbl:hog2}
\end{table}
Following the empirical evaluation in~\cite{Sharon.macbs}, we used the
same three maps from the game \textit{Dragon Age: Origins}
\cite{journals/tciaig/Sturtevant12}. As with the puzzle, 100 random
instances were generated, and the reported numbers are the totals over
the 100 instances. Table~\ref{tbl:hog2} shows the results of CBS,
MA-CBS, MA-CBS/R on 16-agent scenes for a range of values of $B$. 

Again, MA-CBS/R shows the best performance (bold in the table). The
advantage of MA-CBS/R over MA-CBS depends on the map. \textbf{den520d}
consists mostly of open spaces, and MA-CBS/R is \textbf{5 times} faster than
MA-CBS for the tested values of $B$. \textbf{ost003d} is a combination
of open spaces and bottlenecks; MA-CBS/R (for $B=16$) is 10\%
faster than MA-CBS (for $B=1$) for the tested values of $B$, but for
intermediate values of $B$, 16 and 64, MA-CBS/R is \textbf{4 times} faster than MA-CBS: hence the best value
that would be estimated from test runs of the low-level search on
single and combined agents might give a 4-fold increase in
performance. \textbf{brc202d} mostly consists of narrow paths
resulting in many bottlenecks. MA-CBS/R is $\approx 12\%$ faster than
MA-CBS for the tested values of $B$. For all maps and for all values
of $B$, MA-CBS/R is faster than MA-CBS. Moreover, the search time of
MA-CBS for intermediate values of $B$ is often longer than for
extreme (either low or high) values, an evidence which further
supports the advantage of restarting the search upon a merge. Delayed
or Randomized MA-CBS/R can be used to further improve the
performance for the best found value of $B$.

\section{Summary and Future Work}

This paper has several contributions:
\begin{itemize}
\item We provided a justification for the use of a fixed threshold
for decision-making in MA-CBS; the justification was based on the
worst-case analysis of a two-agent MAPF problem. 
\item Using a model problem based on this justification we introduced
  a more efficient version of MA-CBS, MA-CBS/R, where the search is
  restarted after a merge. MA-CBS/R exhibits shorter search times and
  lower numbers of expanded nodes than MA-CBS on both $4\times 4$ partial
  sliding tile puzzle and computer game scenes.
\item We also introduced two improved variants of MA-CBS/R, Randomized
  MA-CBS/R based on a known randomized algorithm, and Delayed MA-CBS/R
  based on the analysis of decision utilities. Both algorithms show
  better yet performance compared to MA-CBS/R.
\end{itemize}
There is room for further improvement of MA-CBS variants. Firstly, the
decision to merge a pair of agents can be made based on the history of
conflict occurrence and resolution through splitting, rather than just
the number of conflicts. Secondly, the tie-breaking, such as selection
of the conflicting agents to split on or merge, and of the conflict to
resolve in case of a split, can be improved using heuristic decision
rules. We believe that metareasoning
techniques~\cite{RussellWefald,Russell.rationality} can be applied
successfully to MAPF domain in general and MA-CBS variants in
particular to design the heuristics.

\bibliographystyle{aaai}
\bibliography{refs}

\end{document}